\documentclass{article}
\pdfoutput=1

\usepackage[nonatbib,preprint]{neurips_2019}

\usepackage[numbers]{natbib}

\usepackage{mathptmx}      

\usepackage[utf8]{inputenc} 
\usepackage[T1]{fontenc}    
\usepackage{graphicx}
\usepackage{amsmath}
\usepackage{amssymb}
\usepackage{amsthm}
\usepackage{mathtools}
\usepackage{subfigure}
\usepackage{comment}

\theoremstyle{plain}
\newtheorem{lemma}{Lemma}
\newtheorem{corollary}{Corollary}

\providecommand{\Paren}[1]{\ensuremath{\left( #1 \right)}}

\providecommand{\defeq}[0]{\stackrel{\text{def}}{=}}

\providecommand{\cP}{\mathcal{P}}

\providecommand{\EE}[2]{\ensuremath{\mathbb{E}_{#1}\left[ #2 \right]}}

\DeclareMathOperator{\softmin}{soft-min}
\DeclareMathOperator{\softmax}{soft-max}
\DeclareMathOperator{\softplus}{soft-plus}

\def\tr{^\top}

\begin{document}

\title{Fully Variational Noise-Contrastive Estimation}

\author{%
  Christopher Zach  \\
  Chalmers University of Technology\\
  Gothenburg, Sweden \\
  \texttt{zach@chalmers.se}
}

\maketitle

\begin{abstract}
  By using the underlying theory of proper scoring rules, we design a family
  of noise-contrastive estimation (NCE) methods that are tractable for latent
  variable models. Both terms in the underlying NCE loss, the one using data
  samples and the one using noise samples, can be lower-bounded as in
  variational Bayes, therefore we call this family of losses \emph{fully
    variational noise-contrastive estimation}. Variational autoencoders are a
  particular example in this family and therefore can be also understood as
  separating real data from synthetic samples using an appropriate
  classification loss. We further discuss other instances in this family of
  fully variational NCE objectives and indicate differences in their empirical
  behavior.
\end{abstract}

\section{Introduction}

Estimating the parameters of a model distribution from a training set is an
important research topic with applications in deep generative models
(e.g.~\cite{goodfellow2014generative,mirza2014conditional,rezende2015variational,kingma2018glow,song2020score,dhariwal2021diffusion}),
out-of-distribution (OOD) or anomaly
detection~\cite{zenati2018adversarially,ren2019likelihood,kirichenko2020normalizing,liu2020energy}
and representation
learning~\cite{dayan1995helmholtz,radford2015unsupervised,chen2016infogan,oord2018representation}.
Maximum-likelihood estimation is the method of choice when the parametric
model distribution is normalized and can be evaluated efficiently (which is
the case for ``elementary'' probability distributions and for normalizing
flows~\cite{rezende2015variational}). The expressiveness of a model
distribution can be enhanced by introducing latent variables and by using an
unnormalized distribution (also known as energy-based model). Both of these
modifications prevent the maximum likelihood method from being applicable:
latent variables often lead to intractable integrals or sums when computing
the marginal likelihood, and likewise the normalization factor (also called
the partition function) of an unnormalized model is typically intractable.

Latent variables are usually addressed by utilizing the evidence lower bound
(ELBO) of the likelihood as in variational
Bayes~(e.g.~\cite{jordan1999introduction}), and parameters of unnormalized
models can be estimated from data by methods such as score
matching~\cite{hyvarinen2005estimation} or noise-contrastive estimation
(NCE,~\cite{gutmann2010noise,gutmann2012noise}). NCE can intuitively be
understood as learning a binary classifier separating training data from
samples drawn from a fully known noise distribution. Variational
NCE~\cite{rhodes2019variational} aims to enable the estimation of unnormalized
latent variable models from data by leverging the ELBO. It succeeds only
partially, since the ELBO cannot be applied on all terms in the NCE objective,
and an intractable marginal remains. In this work we derive modified instances
of NCE that allow the application of the ELBO on all terms, and the resulting
objective is therefore free from intractable sums (or integrals). We call the
resulting method \emph{fully variational noise-contrastive
  estimation}. Interestingly, variational
autoencoders~\cite{kingma2014auto,rezende2014stochastic} are one particular
(and important) instance in this family of fully variational NCE methods.

\section{Background}

\paragraph{Proper scoring rules}
Let $\cP\subseteq \mathbb{R}^d$, and let $G:\cP \to \mathbb{R}$ be a
differentiable convex mapping. The Bregman divergence between
$p\in\cP$ and $q\in\cP$ is defined as
\begin{align}
  D_G(p \| q) \defeq G(p) - \big( G(q) + (p-q)\tr \nabla G(q) \big),
\end{align}
i.e.\ $D_G(p \| q)$ is the error between $G(p)$ and the linearization
(first-order Taylor expansion) of $G$ at $q$. Convexity of $G$ implies that
$D_G(p \| q)$ is non-negative. If $G$ is strictly convex, then
$D_G(p \| q)=0$ iff $p=q$.

Now let $p$ and $q$ be the parameters of a categorical distribution, i.e.\
$P(X=k|p)=p_k$ and $P(X=k|q)=q_k$ for a categorical random variable
$X$ with values in $\{1,\dotsc,d\}$. The domain $\cP$ is therefore the
probability simplex, $\cP = \{ p\in[0,1]^d: \sum_{k=1^d} p_k = 1\}$. In this
setting $D_G(p \| q)$ can be stated as
\begin{align}
  D_G(p \| q) &= G(p) - G(q) + \EE{X\sim p}{\tfrac{d}{dq_X} G(q)} - \EE{X\sim q}{\tfrac{d}{dq_X} G(q)} \nonumber \\
  {} &= G(p) + \EE{X\sim p}{\tfrac{d}{dq_X} G(q) - G(q) - \EE{X'\sim q}{\tfrac{d}{dq_{X'}} G(q)}}.
\end{align}
Minimizing $D_G(p \| q)$ w.r.t.\ $q$ for fixed $p$ is equivalent to
\begin{align}
  \arg\min_{q\in\cP} D_G(p \| q) &= \arg\min_{q\in\cP} -G(q) - \sum\nolimits_k (p_k-q_k) \tfrac{\partial}{\partial q_k} G(q) \nonumber \\
  {} &= \arg\max_{q \in \cP} \EE{X\sim p}{\tfrac{\partial}{\partial q_X} G(q) + G(q) - \sum\nolimits_k q_k \tfrac{\partial}{\partial q_k} G(q) } \nonumber\\
  {} &= \arg\max_{q \in \cP} \EE{X\sim p}{ S(X, q) },
\end{align}
where we defined the \emph{proper scoring rule} (PSR) $S$ as follows,
\begin{align}
  S(x,q) \defeq \tfrac{\partial}{\partial q_x} G(q) + G(q) - \sum\nolimits_k q_k \tfrac{\partial}{\partial q_k} G(q).
  \label{eq:PSR}
\end{align}
Note that maximization w.r.t.\ $q$ only requires samples from $p$, but does
not need the knowledge of the distrbution $p$ itself. Therefore proper scoring
rules are one method to estimate distribution parameters when only samples
from an unknown data distribution $p$ are available.

If $G$ is strictly convex, then the resulting PSR is a \emph{strictly} PSR.
If e.g.\ $G$ is chosen as the negated Shannon entropy, then
$S(x,q)=\log q_x$ is called the \emph{logarithmic} scoring rule underlying
maximum likelihood estimation and the cross-entropy loss in machine
learning. It is an instance of a \emph{local} PSR~\cite{parry2012proper},
which does not depend on any value of $q_{x'}$ for $x'\ne x$ (the score
matching cost~\cite{hyvarinen2005estimation} being another example). We refer
to~\cite{gneiting2007strictly} and~\cite{dawid2014theory} for an extensive
overview and further examples of proper scoring rules.

\paragraph{PSRs for binary RVs}
When $X$ is a binary random variable, and therefore $x\in\{0,1\}$, then we
only need one parameter $\mu\in[0,1]$ to characterize the corresponding
Bernoulli distribution. For a differentiable convex function
$G:[0,1]\to\mathbb{R}$ the induced Bregman divergence between
$\mu\in[0,1]$ and $\nu\in[0,1]$ is given by
\begin{align}
  D_G(\mu\|\nu) = G(\mu) - G(\nu) - (\mu-\nu) G'(\nu)
\end{align}
and
\begin{align}
  \arg\min_{\nu\in[0,1]} D_G(\mu\|\nu) &= \arg\max_{\nu\in[0,1]} G(\nu) + (\mu-\nu) G'(\nu) \nonumber \\
  {} &= \arg\max_{\nu\in[0,1]} \EE{x\sim\text{Ber}(\mu)}{G(\nu) + (x-\nu) G'(\nu)}.
\end{align}
The resulting PSR $S$ is therefore
\begin{align}
  \label{eq:bin_PSR}
  S(1,\nu) &= G(\nu) + (1-\nu) G'(\nu) & S(0,1\!-\!\nu) &= G(\nu) - \nu G'(\nu).
\end{align}
$G$ can be recovered via
\begin{align}
  G(\nu) = \nu S(1,\nu) + (1\!-\!\nu) S(0,1\!-\!\nu) = \EE{x\sim\text{Ber}(\mu)}{S(x,x\nu + (1\!-\!x)(1\!-\!\nu))}.
  \label{eq:G_from_S}
\end{align}

\paragraph{Noise-contrastive estimation}
Noise-contrastive estimation (NCE,~\cite{gutmann2010noise,gutmann2012noise})
ultimately casts the estimation of parameters of an unknown data distribution
as a binary classification problem. Let
$\Omega \subseteq \mathbb{R}^n$ and $X$ be a $n$-dimensional random vector.
Let $p_d$ the (unknown) data distrbution, $p_\theta$ a model distribution
(with parameters $\theta$) and $p_n$ a user-specified noise distribution. Let
$Z$ be a (fair) Bernoulli RV that determines whether a sample is drawn from
the data (respectively model) distribution or from the noise distribution
$p_d$.\footnote{We omit the possibility of using general Bernoulli RV for
  notational simplicity.} NCE applies the logarithmic PSR to match the
posteriors,
\begin{align}
    P_{d,n}(Z=1 | X=x) &= \frac{p_d(x)}{p_d(x) \!+\! p_n(x)}
    & P_{\theta,n}(Z=1 | X=x) &= \frac{p_\theta(x)}{p_\theta(x) \!+\! p_n(x)},
\end{align}
which yields the NCE objective
\begin{align}
  J_{\text{NCE}}(\theta) = \EE{X\sim p_d}{\log \frac{p_\theta(X)}{p_\theta(X) + p_n(X)}}
  + \EE{X\sim p_n}{\log \frac{p_n(X)}{p_\theta(X) + p_n(X)}}.
  \label{eq:NCE}
\end{align}
After introducing $r_\theta(x) \defeq p_\theta(x)/p_n(x)$ this reads as
\begin{align}
  J_{\text{NCE}}(\theta) = \EE{X\sim p_d}{-\log \Paren{1 + r_\theta(X)^{-1}}} + \EE{X\sim p_n}{-\log \big( 1+r_\theta(X) \big)},
  \label{eq:NCE_r}
\end{align}
establishing the connection to logistic regression.
At first glance this is superficially similar to
GANs~\cite{goodfellow2014generative}, but it lacks e.g.\ the problematic
min-max structure of GANs.
In contrast to e.g.\ maximum likelihood estimation, NCE is applicable even
when the model distribution is unnormalized, i.e.
\begin{align}
  p_\theta(x) = \tfrac{1}{Z(\theta)} p_\theta^0(x)
\end{align}
for an unnormalized model $p_\theta^0(x)$ and an intractable partition
function $Z(\theta) = \sum_x p_\theta^0(x)$.\footnote{For brevity we use sums
  to refer to marginalization of RV, but these sums should always be
  understood as the appropriate Lebesque integrals.}
NCE allows to estimate the value of the partition function $Z(\theta)$ for the
obtained model parameters $\theta$ by augmenting the parameter vector to
$(\theta,Z)$ and use the relation $p_\theta(x) = p_\theta^0(x)/Z$. Extensions
to the basic NCE framework are discussed in~\cite{pihlaja2010family} and
\cite{ceylan2018conditional}.

NCE is not directly applicable to latent variable models, where the joint
density $p_\theta(X,Z)$ is specified, but the induced marginal
$p_\theta(X)$ is only indirectly given via
\begin{align}
  p_\theta(x) = \sum\nolimits_z p_\theta(x,z) = \sum\nolimits_z p_\theta(x|z) p_Z(z),
\end{align}
where we use a generative model for the joint $p_\theta(X,Z)$.

Using latent variable models greatly enhances the expressiveness of model
distributions, but exact computation of the marginal $p_\theta(x)$ is often
intractable. By noting that the term under the first expectation in
Eq.~\ref{eq:NCE_r} is concave w.r.t.\ $r_\theta(x)$, Variational
NCE~\cite{rhodes2019variational} proposes to apply the evidence lower bound
(ELBO) to obtain a tractable variational lower bound for the first term in
Eq.~\ref{eq:NCE_r}. Unfortunately, the second term in Eq.~\ref{eq:NCE_r} is
convex in $r_\theta$ and the ELBO does not apply here. Importance sampling is
leveraged instead to estimate the intractable expectation inside the second
term. In the following section we show how the ELBO can be applied on both
terms in a slightly generalized version of NCE.

\section{Fully Variational NCE}

First, we generalize the NCE objective (Eq.~\ref{eq:NCE}) to arbitrary
strictly proper scoring rules for binary random variables,
\begin{align}
  J_{S\text{-NCE}}(\theta) &= \EE{x\sim p_d}{S\Paren{1, \tfrac{r_\theta(x)}{1 + r_\theta(x)}}} + \EE{x\sim p_n}{S\Paren{0, \tfrac{1}{1+r_\theta(x)} }},
  \label{eq:S_NCE}
\end{align}
where $r_\theta$ is the density ratio,
$r_\theta(x) \defeq p_\theta(x)/p_n(x)$.
$J_{S\text{-NCE}}$ is maximized w.r.t.\ the parameters $\theta$ in this
formulation. Recall that $r_\theta(x)/(1+r_\theta(x))$ is the posterior of
$x$ being a sample drawn from the model $p_\theta$, and
$1/(1+r_\theta(x))$ is the posterior for $x$ being a noise sample. Our aim is
to determine a convex function $G$ such that both mappings
\begin{align}
  \label{eq:f_and_S_relation}
  f_1(r) = S(1, r/(1+r)) & & \text{and} & & f_0(r) = S(0, 1/(1+r))
\end{align}
are concave. If this is the case, then
\begin{align}
  f_k\big( r_\theta(x) \big) &= f_k\Paren{ \frac{p_\theta(x)}{p_n(x)} } = f_k\Paren{ \frac{\sum_z p_\theta(x,z)}{p_n(x)} }
                               = f_k\Paren{ \frac{\sum_z p_\theta(x,z) q_k(z|x)}{p_n(x)q_k(z|x)} } \nonumber \\
  {} &\ge \sum\nolimits_z q_k(z|x) f_k\Paren{ \frac{p_\theta(x,z)}{p_n(x)q_k(z|x)} }
       = \EE{z\sim q_k(Z|x)}{f_k\Paren{ \frac{p_\theta(x,z)}{p_n(x)q_k(z|x)}}} \nonumber
\end{align}
for $k\in\{0,1\}$. $q_k(Z|X)$ is a posterior corresponding to the encoder
part. Overall, $J_{S\text{-NCE}}$ in Eq.~\ref{eq:S_NCE} can be lower bounded
as follows,
\begin{align}
  J_{S\text{-NCE}}(\theta) \!=\! \EE{x\sim p_d}{f_1(r_\theta(x))} \!+\! \EE{x\sim p_n}{f_0(r_\theta(x))}
  \ge \max_{q_1,q_0} J_{S\text{-fvNCE}}(\theta,\!q_1,\!q_0)
  \label{eq:double_ELBO}
\end{align}
with the r.h.s.\ defined as the \emph{fully variational} NCE loss,
\begin{align}
  J_{S\text{-fvNCE}}(\theta,q_1,q_0) \defeq \EE{x\sim p_d,z\sim q_1(Z|x)}{f_1\Paren{ \frac{p_\theta(x,z)}{p_n(x)q_1(z|x)}}}
  + \EE{x\sim p_n,z\sim q_0(Z|x)}{f_0\Paren{ \frac{p_\theta(x,z)}{p_n(x)q_0(z|x)}}}.
  \label{eq:fvNCE}
\end{align}
Note that we allow in principle two separate encoders, $q_1$ and
$q_0$, since the ELBO is applied at two places independently. For brevity we
introduce the following short-hand notations for the joint distributions,
\begin{align}
  p_{d,k}(x,z) \defeq p_d(x) q_k(z|x) & & p_{n,k}(x,z) \defeq p_n(x) q_k(z|x),
\end{align}
resulting in a more compact expression for $J_{S\text{-fvNCE}}$,
\begin{align}
  J_{S\text{-fvNCE}}(\theta,\!q_1,\!q_0) = \EE{(x,z)\sim p_{d,1}}{f_1\!\Paren{ \frac{p_\theta(x,z)}{p_{n,1}(x,z)}}}
  \!+\! \EE{(x,z)\sim p_{n,0}}{f_0\!\Paren{ \frac{p_\theta(x,z)}{p_{n,0}(x,z)}}}\!.
\end{align}
From $p_\theta(x)p_\theta(z|x) = p_\theta(x,z)$ we deduce that the lower bound
is tight, i.e.\
$J_{S\text{-NCE}}(\theta)=\max_{q_1,q_0} J_{S\text{-fvNCE}}(\theta,q_1,q_0)$
when the encoders $q_1$ and $q_0$ are equal to the model posterior,
$q_1(Z|X)=q_0(Z|X) = p_\theta(Z|X)$ a.e. $J_{S\text{-fvNCE}}$ in
Eq.~\ref{eq:fvNCE} is formulated as a population loss, but the corresponding
empirical risk can be immediately obtained by sampling from $p_d$,
$p_n$ and the encoder distributions.

Now the question is whether such concave mappings $f_1$ and $f_0$ satisfying
Eq.~\ref{eq:f_and_S_relation} for a PSR $S$ exist. Since common PSRs such as
the logarithmic and the quadratic PSR violate these properties,
existence of such a PSR is not obvious. The next section discusses how to
construct such PSRs and provides examples.

\section{A Family of Suitable Proper Scoring Rules}

In this section we construct a pair $(f_1,f_0)$ of concave mappings, such that
the induced functions $S(1,\cdot)$ and $S(0,\cdot)$ in
Eq.~\ref{eq:f_and_S_relation} form a PSR. The following result provides
sufficient conditions on such a pair $(f_1,f_0)$:
\begin{lemma}
  Let a pair of functions $(f_0,f_1)$,
  $f_k\!:\!(0,\infty) \to \mathbb{R}$, satisfy the following:
  \begin{enumerate}
  \item Both $f_1$ and $f_0$ are concave,
  \item $f_1$ and $f_0$ satisfy the compatibility condition
    \begin{align}
      f_0'(r) = -rf_1'(r)
      \label{eq:f_cond}
    \end{align}
    for all $r>0$,
  \item the mapping
    $G(\mu) = \mu f_1(\mu/(1-\mu)) + (1-\mu) f_0(\mu/(1-\mu))$ is convex in
    $(0,1)$.
  \end{enumerate}
  Then $S$ is a PSR. Such pairs $(f_1,f_0)$ are said to have to \emph{double
    ELBO} property.
\end{lemma}
\begin{proof}
  We abbreviate $S_1(\mu):=S(1,\mu)$ and
  $S_0(1\!-\!\mu):=S(0,1\!-\!\mu)$ and recall the relations between
  $S$ and $G$:
  \begin{align}
    \begin{split}
      G(\mu)&=\mu S_1(\mu) + (1\!-\!\mu) S_0(1\!-\!\mu) \\
      S_0(1\!-\!\mu) &= G(\mu)-\mu G'(\mu) \\
      S_1(\mu) &= G(\mu) + (1\!-\!\mu) G'(\mu) = S_0(1\!-\!\mu) + G'(\mu)
    \end{split}
    \label{eq:PSR_1}
  \end{align}
  and therefore $G'(\mu) = S_1(\mu)-S_0(1\!-\!\mu)$. We calculate
  \begin{align}
    G'(\mu) = S_1(\mu) - S_0(1\!-\!\mu) + \mu S_1'(\mu) - (1\!-\!\mu) S_0'(1\!-\!\mu)
  \end{align}
  Combining these relations implies that
  \begin{align}
    \mu S_1'(\mu) - (1\!-\!\mu) S_0'(1\!-\!\mu) = 0 \iff S_0'(1-\mu) = \tfrac{\mu}{1-\mu} \cdot S_1'(\mu)
  \end{align}
  Now the relation between $\mu$ and $r$ is
  $\mu = r/(1+r)$ and therefore $r = \mu/(1-\mu)$,
  which we use to express $(f_1,f_0)$ in terms of $(S_1,S_0)$,
  \begin{align}
    f_1(r) = S_1(\mu) = S_1(r/(1+r)) & & f_0(r) = S_0(1\!-\!\mu) = S_0(1/(1+r)).
  \end{align}
  Using $d\mu/dr = (1+r)^{-2}$ and
  \begin{align}
    f_1'(r) = \tfrac{1}{(1+r)^2} S_1'(\mu) & & f_0'(r) = -\tfrac{1}{(1+r)^2} S_0'(1-\mu) \nonumber,
  \end{align}
  the condition can be restated as
  \begin{align}
    -(1+r)^2 f_0'(r) = r \cdot (1+r)^2 f_1'(r) \iff f_0'(r) = -r f_1'(r),
  \end{align}
  which is the second requirement on $(f_1,f_0)$. Now if $(f_1,f_0)$ satisfy
  Eq.~\ref{eq:f_cond}, then $(S_1,S_0)$ satisfy the relations of a binary PSR
  in Eq.~\ref{eq:PSR_1} for an induced function $G$. If $G$ is now convex,
  then $(S_1,S_0)$ is a PSR. \qed
\end{proof}
One consequence of the condition in Eq.~\ref{eq:f_cond} is, that $f_1$ is
increasing and $f_0$ is decreasing or vice versa. This further implies that
$S$ cannot be symmetric, i.e.
\begin{align}
  S(1,\mu) \ne S(0,1\!-\!\mu),
\end{align}
and positive and negative samples are penalized differently in the overall
loss. This is in contrast to many well-known PSR, which are symmetric (such as
the logarithmic PSR used in NCE). The condition also implies that
\begin{align}
  f_0''(r) =  -f_1'(r) - r f_1''(r) \stackrel{!}\le 0 \nonumber.
\end{align}
Since $f_1$ is concave and $r\ge 0$, $-r f_1''(r)\ge 0$. This has to be
compensated by $f_1'$ increasing sufficiently fast with $r$. Since
$f_1'(r) \ge -rf_1''(r) \ge 0$, $f_1$ is increasing and $f_0$ is decreasing in
$\mathbb{R}_{\ge 0}$. This observation yields some intuition on
$J_{S\text{-fvNCE}}$ in Eq.~\ref{eq:fvNCE}: the first term aims to align
$p_\theta$ with $p_{d,1}$ by maximizing $p_\theta(x,z)/p_{n,1}(x,z)$ for real
data (and its code), whereas the second term favors mis-alignment between
$p_\theta$ and $p_{n,0}$ for noise samples (by minimizing the likelihood ratio
$p_\theta(x,z)/p_{n,0}(x,z)$).

Eq.~\ref{eq:f_cond} immediately allows to establish one pair
$(f_1,f_0)$ satisfying the double ELBO property: we choose
$f_1(r) = \log r$, which yields $f_0'(r)=-1$ and therefore
$f_0(r)=-r$. Both $f_1$ and $f_0$ are concave. Further,
\begin{align}
  S_1(\mu) = \log \tfrac{\mu}{1-\mu} & &  S_0(1\!-\!\mu) = -\tfrac{\mu}{1-\mu}
\end{align}
and therefore
\begin{align}
  G(\mu) &= \mu S_1(\mu) + (1\!-\!\mu) S_0(1\!-\!\mu) = \mu \Paren{ \log \tfrac{\mu}{1-\mu} - 1 },
\end{align}
which is convex in $(0,1)$. Thus, we have established the existence of one PSR
allowing the ELBO being applied on both terms as in
Eq.~\ref{eq:double_ELBO}. This example can be generalized to the following
parametrized family of PSRs:
\begin{lemma}
  \label{lem:family1}
  A family of PSRs satisfying the double ELBO property is given by
  \begin{align}
    f_1(r) = \log(r + \beta) & & f_0(r) = \beta \log(r+\beta) - r
  \end{align}
  for any $\beta \ge 0$,.
\end{lemma}
\begin{proof}
  This follows from
  \begin{align}
    f_0'(r) = -r f_1'(r) = -\tfrac{r}{r+\beta} = - \tfrac{r+\beta-\beta}{r+\beta} = -1 + \tfrac{\beta}{r+\beta}
    \implies f_0(r) = \beta\log (r+\beta) - r \nonumber.
  \end{align}
  Further, $G''$ can be calculated as
  \begin{align}
    G''(\mu) = -\frac{1}{(1-\mu)^2 (\beta\mu - \mu - \beta)} = \frac{1}{(1-\mu)^2 (\mu + \beta(1-\mu))} > 0,
  \end{align}
  which establishes the convexity of $G$ (due to $(1-\mu)^2>0$ and
  $\mu+\beta(1-\mu)>0$ for $\mu\in(0,1)$ and $\beta\ge 0$). \qed
\end{proof}
A 2-parameter family of PSRs is given next.
\begin{lemma}
  \label{lem:family2}
  For $\alpha \in (0,1]$ and $\beta\ge 0$ we choose
  \begin{align}
    f_1(r) &= \tfrac{1}{\alpha} (r+\beta)^\alpha & & f_0(r) = -\tfrac{1}{\alpha+1} (r+\beta)^{\alpha+1} \nonumber.
  \end{align}
  This pair induces a strictly PSR satisfying the double ELBO property.
\end{lemma}
\begin{proof}
  Both $f_1$ and $f_0$ are clearly concave. We deduce
  \begin{align}
    f_1'(r) = (r+\beta)^{\alpha-1} & & f_0'(r) = -r(r+\beta)^{\alpha-1} = -r f_1'(r),
  \end{align}
  hence $(f_1,f_0)$ satisfy the condition in Eq.~\ref{eq:f_cond}.
  $G''(\mu)$ can be calculated as
  \begin{align}
    G''(\mu) &= \Paren{\frac{\mu + \beta(1-\mu)}{1-\mu}}^\alpha \cdot \frac{\mu + \beta (2-\alpha)(1-\mu)}{(1-\mu)^2 (\mu + \beta(1-\mu))^2}.
  \end{align}
  The first factor is positive for $\alpha\in(0,1]$, $\beta\ge 0$ and
  $\mu\in(0,1)$. Analogously, the second factor is positive since the
  numerator is positive for the allowed values of
  $(\mu,\alpha,\beta)$, and the denominator is a product of squares. \qed
\end{proof}
Since
\begin{align}
  \lim_{\alpha\to 0^+} f_0'(r; \alpha, \beta) = -1 \implies \lim_{\alpha\to 0^+} f_1'(r; \alpha, \beta) = (r+\beta)^{-1},
\end{align}
we deduce that the limit $\alpha\to 0^+$ yields the pair $(f_1,f_0)$ from
Lemma~\ref{lem:family1} (up to constants independent of $r$).

\begin{figure}[tb]
  \centering
  \subfigure[$S_1$]{\includegraphics[width=0.48\textwidth]{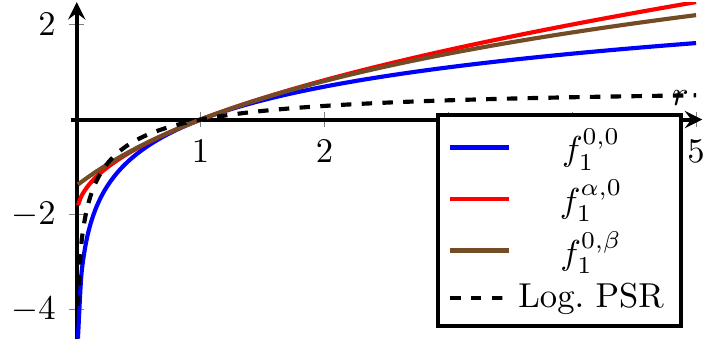}}
  \subfigure[$S_0$]{\includegraphics[width=0.48\textwidth]{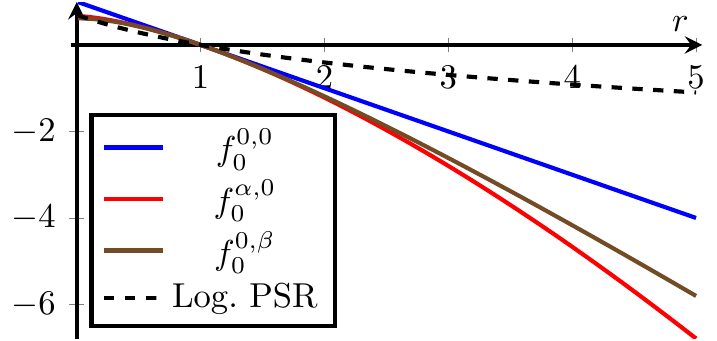}}
  \caption{Several pairs $(f_1,f_0)$, in particular
    $(f_1^{0,0},f_0^{0,0})$, $(f_1^{\alpha,0},f_0^{\alpha,0})$ and
    $(f_1^{0,\beta},f_0^{0,\beta})$ for $\alpha=1/2$ and $\beta=1$ (solid
    curves). Both $f_1$ and $f_0$ are concave functions. The pair
    $(f_1,f_0)$ induced by the logarithmic PSR is shown for reference (dashed
    curve, which is concave in (a), but convex in (b)). }
  \label{fig:f_pairs}
\end{figure}

For visualization purposes it is convenient to normalize $f_1$ and
$f_0$ such that $f_1(1)=f_0(1)=0$ and $f_1'(1)=1$ (and therefore
$f_0'(1)=-1$). With such normalization the above pairs are given by
\begin{align}
  \begin{split}
    f_1(r; \alpha,\beta) &
    = \tfrac{(1+\beta)^{1-\alpha}}{\alpha} \big( (r+\beta)^\alpha - (1+\beta)^\alpha \big) \\
    f_0(r; \alpha,\beta) &= -\tfrac{(1+\beta)^{1-\alpha}}{\alpha(\alpha+1)}
    \big( (\alpha r-\beta) (r+\beta)^\alpha - (\alpha-\beta) (1+\beta)^\alpha \big).
  \end{split}
\end{align}
Few instances of $(f_1^{\alpha,\beta},f_0^{\alpha,\beta})$ are depicted in
Fig.~\ref{fig:f_pairs}. We further introduce the fully variational NCE loss
parametrized by $(\alpha,\beta)$,
\begin{align}
  J^{\alpha,\beta}_{\text{fvNCE}}(\theta,q_1,q_0) &\defeq \EE{(x,z)\sim p_{d,1}}{f_1\Paren{ \frac{p_\theta(x,z)}{p_{n,1}(x,z)}; \alpha,\beta}}
                                                    + \EE{(x,z)\sim p_{n,0}}{f_0\Paren{ \frac{p_\theta(x,z)}{p_{n,0}(x,z)}; \alpha,\beta}}.
\end{align}
We would like to get a better understanding of these PSRs in terms of losses
used for binary classification. Recall that
\begin{align}
  r = \tfrac{p}{q} & & \mu = \tfrac{p}{p+q} = \tfrac{1}{1+1/r} = \tfrac{r}{1+r} = \sigma(\Delta)
  & & r = \tfrac{\mu}{1-\mu} = \tfrac{\sigma(\Delta)}{1-\sigma(\Delta)} = \tfrac{\sigma(\Delta)}{\sigma(-\Delta)} = \exp(\Delta) \nonumber.
\end{align}
Here $\Delta$ is the logit of the binary classifier. We minimize a
classification loss, hence we consider the negated PSRs. Thus, we obtain for
the logarithmic PSR,
\begin{align}
  -\log(\mu) &= -\log(\sigma(\Delta)) = \log(1+\exp(-\Delta)) = \softplus(-\Delta) \nonumber \\
  -\log(1-\mu) &= -\log(1-\sigma(\Delta)) = -\log(\sigma(-\Delta)) = \softplus(\Delta) \nonumber,
\end{align}
where $\softplus(u)\defeq \log(1+e^u)$. Inserting
$f_1(r)=\log(r+\beta)$ and $f_0(r)=\beta\log(r+\beta) - r$ yields
\begin{align}
  -f_1(r) &= -\!\log(r\!+\!\beta) = -\log(e^\Delta\!+\!\beta) \doteq -\!\softmax(\Delta,\log\beta) = \softmin(-\Delta, -\!\log\beta) \nonumber \\
  -f_0(r) &= r-\beta\log(r\!+\!\beta) = e^\Delta + \beta \softmin(-\Delta, -\!\log\beta) \nonumber
\end{align}
Finally, $f_1(r)=r^\alpha/\alpha$, $f_0(r)=-r^{\alpha+1}/(\alpha+1)$ results in
\begin{align}
  -f_1(r) &= -\tfrac{1}{\alpha} r^\alpha = -\tfrac{1}{\alpha} e^{\alpha\Delta}
  & -f_0(r) &= \tfrac{1}{\alpha+1} r^{\alpha+1} = \tfrac{1}{\alpha+1} e^{(\alpha+1)\Delta} \nonumber.
\end{align}
Graphically, the difference between the logistic classification loss and the
double-ELBO losses is, that the logistic loss solely penalizes incorreect
predictions and the double ELBO losses strongly favor true positives instead
(as shown in Fig.~\ref{fig:losses}).

\begin{figure}[tb]
  \centering
  \subfigure[Class 1]{\includegraphics[width=0.48\textwidth]{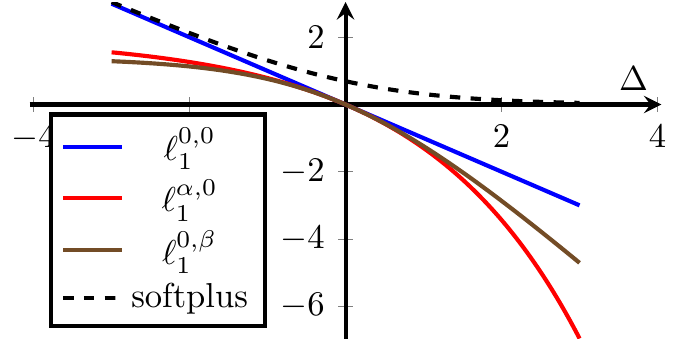}}
  \subfigure[Class 0]{\includegraphics[width=0.48\textwidth]{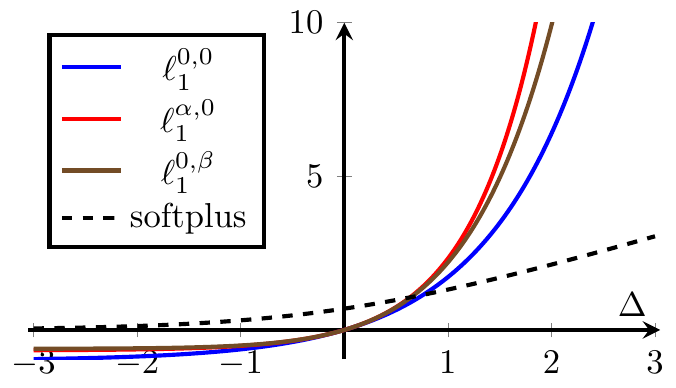}}
  \caption{The PSRs from Fig.~\ref{fig:f_pairs} reinterpreted as binary
    classification losses in terms of log-ratios $\Delta=\log r$. The
    soft-plus loss corresponds to the logarithmic PSR. }
  \label{fig:losses}
\end{figure}

We conclude this section by a noting that non-negative linear combinations of
double ELBO pairs have the double ELBO property as well:
\begin{corollary}
  \label{cor:closedness}
  The set of pairs with the double ELBO property is a convex cone.
\end{corollary}
This follows from the linearity of the relations Eq.~\ref{eq:f_cond} and
Eq.~\ref{eq:G_from_S}.

\section{Instances of Fully Variational NCE}

In this section we discuss several instances of
$J^{\alpha,\beta}_{\text{fvNCE}}$ for specific choices of $\alpha$ and
$\beta$. For easier identification of known frameworks we focus on normalized
model distributions $p_\theta$, but the extension to unnormalized models is
straightforward.

\subsection{Variational auto-encoders: $(\alpha,\beta)=(0,0)$}
\label{sec:J00}

We choose $(\alpha,\beta)=(0,0)$ in the 2-parameter family given in
Lemma~\ref{lem:family2}, i.e.\ $f_1(r)=\log r$ and $f_0(r)=-r$. The resulting
fully variational NCE objective therefore is given by
\begin{align}
  J^{0,0}_{\text{fvNCE}}(\theta,\!q_1,\!q_0) = \EE{(x,z)\sim p_{d,1}}{\log \Paren{ \frac{p_\theta(x,z)}{p_{n,1}(x,z)}}}
  \!-\! \EE{(x,z)\sim p_{n,0}}{\frac{p_\theta(x,z)}{p_{n,0}(x,z)}}\!.
  \label{eq:J00}
\end{align}
We first focus on the second term:
\begin{align}
  \EE{(x,z)\sim p_{n,0}}{\frac{p_\theta(x,z)}{p_{n,0}(x,z)}} = \sum_{x,z:p_{n,0}(x,z)>0} p_\theta(x,z) \le 1.
  \label{eq:J00_term2}
\end{align}
Now if
$\operatorname{supp}(p_\theta) \subseteq \operatorname{supp}(p_{n,0})$, then
the r.h.s.\ of Eq.~\ref{eq:J00_term2} is exactly~1, otherwise it is bounded
by~1 from above.\footnote{If we use unnormalized models $p_\theta^0$, then
  Eq.~\ref{eq:J00_term2} is bounded by $Z(\theta)$.} We assume that
$\operatorname{supp}(p_\theta) \subseteq \operatorname{supp}(p_{n,0})$, then
the last term in Eq.~\ref{eq:J00} is~1, and since
$\EE{x\sim p_d}{\log p_n(x)}$ is constant, we obtain
\begin{align}
  J^{0,0}_{\text{fvNCE}}(\theta,q_1,q_0) \doteq \EE{x\sim p_d, z\sim q_1(Z|x)}{\log \Paren{ \frac{p_\theta(x,z)}{q_1(z|x)}}}.
\end{align}
After factorizing $p_\theta(x,z)=p_\theta(x|z)p_Z(z)$ this can be identified
as the variational autoencoder loss (up to constants independent of
$\theta$ and $q_1$),
\begin{align}
  J^{0,0}_{\text{fvNCE}}(\theta,q_1)
  \doteq \underbrace{\EE{x\sim p_d}{\EE{z\sim q_1(Z|x)}{\log p_\theta(x|z)} - D_{KL}(q_1(Z|x)\|p_Z)}}_{\defeq J_{\text{VAE}}(\theta,q_1)}.
\end{align}
Thus, in this setting standard VAE training can be understood as
variance-reduced implementation of $J^{0,0}_{\text{fvNCE}}$ (since the
stochastic second term becomes a closed-form constant). If
$\operatorname{supp}(p_\theta) \not\subseteq \operatorname{supp}(p_{n,0})$,
then
\begin{align}
  -\EE{(x,z)\sim p_{n,0}}{\frac{p_\theta(x,z)}{p_{n,0}(x,z)}} \ge -1
\end{align}
and optimizing the VAE loss $J_{\text{VAE}}$ is maximizing a lower bound of
$J^{0,0}_{\text{fvNCE}}$. Now let $q_0(Z|X)$ be a deterministic encoder, i.e.\
$q_0(z|x)=\mathbf{1}[z=g_0(x)]$. In this setting
\begin{align}
  J^{0,0}_{\text{fvNCE}}(\theta,q_1,q_0) &\doteq J_{\text{VAE}}(\theta,q_1) - \EE{x\sim p_n}{\frac{p_\theta(x,g_0(x))}{p_{n}(x)}} \nonumber \\
  {} &= J_{\text{VAE}}(\theta,q_1) - \sum\nolimits_x p_\theta(x,g_0(x)).
\end{align}
Intuitively, $J^{0,0}_{\text{fvNCE}}$ aims to autoencode real data well, but
at the same time prefers poor reconstructions for arbitrary
inputs. $J^{0,0}_{\text{fvNCE}}$ uses importance weighting to estimate
$\sum_x p_\theta(x,g_0(x))$. This term only becomes relevant in the objective
if the two encoders $q_1$ and $q_0$ are tied in some way (otherwise
$g_0$ may map the input to a constant code that is unlikely to be sampled from
$q_1$).

It is interesting to note that deterministic (and tied) encoders yield
somewhat different objectives when comparing classical autoencoders, VAEs and
the fully variational NCE:
\begin{align}
  J_{\text{AE}}(\theta,g) &= \EE{x\sim p_d}{\log p_\theta(x | g(x))} \\
  J_{\text{VAE}}(\theta,g) &= J_{\text{AE}}(\theta,g) + \EE{x\sim p_d}{\log p_Z(g(x))} - \gamma \\
  J^{0,0}_{\text{fvNCE}}(\theta,g) &= J_{\text{VAE}}(\theta,g) - \sum\nolimits_x p_\theta(x,g(x)), \label{eq:rVAE}
\end{align}
where $\gamma := \max_z \log p_Z(z)$ is introduced to ensure
$\log p_Z(z) - \gamma \le 0,$\footnote{This is only necessary for continuous
  latent variables as pmf's are always in $[0,1]$.} which allows us to obtain
the following chain of inequalities,
\begin{align}
  J_{\text{AE}}(\theta,g) \ge J_{\text{VAE}}(\theta,g) \ge J^{0,0}_{\text{fvNCE}}(\theta,g).
\end{align}
$J^{0,0}_{\text{fvNCE}}$ can be also interpreted as a well-justified instance
of regularized autoencoders~\cite{ghosh2020variational}. When using tied
stochastic encoders $q_0=q_1$ satisfying
$\operatorname{supp}(p_\theta) \subseteq \operatorname{supp}(p_{n,0})$, using
the empirical version the 2nd expectation in Eq.~\ref{eq:J00} (instead of
dropping it due to being a constant) can be beneficial in scenarios explicitly
requiring poor reconstruction of certain inputs. The downside is a higher
variance in the empirical loss and its gradients. Overall, a variational
autoencoder can be generally understood as variance-reduced instance of fully
variational NCE.

\subsection{``Robustified'' VAEs: $(\alpha,\beta)=(0,1)$}
\label{sec:J01}

Now we consider the pair $f_1(r)=\log(1+r)$ and $f_0(r)=\log(1+r)-r$. We read
\begin{align}
  \begin{split}
    J^{0,1}_{\text{fvNCE}}(\theta,q_1,q_0) &= \EE{(x,z)\sim p_{d,1}}{\log\Paren{1 + \frac{p_\theta(x,z)}{p_{n,1}(x,z)}}} \\
    {} &+ \EE{(x,z)\sim p_{n,0}}{\log\Paren{ 1 + \frac{p_\theta(x,z)}{p_{n,0}(x,z)}} - \frac{p_\theta(x,z)}{p_{n,0}(x,z)} }.
  \end{split}
         \label{eq:J01}
\end{align}
We assume
$\operatorname{supp}(p_\theta) \!\subseteq\!\operatorname{supp}(p_{n,0})$,
then the 3rd term can be dropped (see Sec.~\ref{sec:J00}). With tied encoders
$q\!=\!q_1\!=\!q_0$ we arrive at a near-symmetric cost
\begin{align}
  &J^{0,1}_{\text{fvNCE}} \doteq \EE{(x,z)\sim p_{d,1}}{\log\Paren{1 \!+\! \frac{p_\theta(x,z)}{p_{n,1}(x,z)}}}
                           + \EE{(x,z)\sim p_{n,0}}{\log\Paren{ 1 \!+\! \frac{p_\theta(x,z)}{p_{n,0}(x,z)}} } \nonumber \\
  {} &= \EE{(x,z)\sim p_{d,1}}{\log\Paren{1 \!+\! \frac{p_\theta(x,z)}{p_n(x) q(z|x)}} }
       + \EE{(x,z)\sim p_{n,1}}{\log\Paren{1 \!+\! \frac{p_\theta(x,z)}{p_n(x) q(z|x)}} } \nonumber \\
  {} &= \EE{(x,z)\sim p_{d,1}}{\softplus(\Delta(x,z) } + \EE{(x,z)\sim p_{n,1}}{\softplus(\Delta(x,z) },
\end{align}
where we introduced the shorthand notation
$\Delta(x,z) = \log p_\theta(x,z) - \log p_n(x)-\log q(z|x)$. This lower
bound is tight if $q(z|x)=p_\theta(z|x)=p_\theta(x,z)/p_\theta(x)$. In this
case the ratio inside the log simplifies to
\begin{align}
  \frac{p_\theta(x,z)}{p_n(x) q(z|x)} = \frac{p_\theta(x,z) p_\theta(x)}{p_n(x) p_\theta(x,z)} = \frac{p_\theta(x)}{p_n(x)}
\end{align}
and $\Delta(x,z) = \log p_\theta(x) - \log p_n(x)$.  Note that
$\log p_n(x)$ is expected to be small for real samples $x$ and large for noise
samples. $J^{0,1}_{\text{fvNCE}}$ can be interpreted as a version of VAEs
aiming to reconstruct both real and noise samples well, but is based on a
robustified reconstruction error (but with different and sample dependent
truncation values for real and noise samples). In practice this cost appears
to behave similar to AEs and VAEs (see Sec.~\ref{sec:eval_alpha} and
Table~\ref{tab:alphas}).

\subsection{Weighted squared distance: $(\alpha,\beta)=(1,0)$}

As a last example we consider $f_1(r)=r$ and $f_0(r)=-r^2/2$:
\begin{align}
  \label{eq:J10}
  J^{1,0}_{\text{fvNCE}}(\theta,q_1,q_0) &= \EE{(x,z)\sim p_{d,1}}{\frac{p_\theta(x,z)}{p_{n,1}(x,z)}}
                                           \!-\! \frac{1}{2} \EE{(x,z)\sim p_{n,0}}{\Paren{ \frac{p_\theta(x,z)}{p_{n,0}(x,z)}}^2}
\end{align}
Note that the encoder $q_1$ cancels in the first term, as
\begin{align}
  \EE{(x,z)\sim p_{d,1}}{\frac{p_\theta(x,z)}{p_{n,1}(x,z)}} = \sum_{x,z} \frac{p_d(x) q_1(z|x) p_\theta(x,z)}{p_n(x) q_1(z|x)}
  = \EE{\substack{x\sim p_d\\z\sim p_Z}}{\frac{p_\theta(x|z)}{p_n(x)}}.
\end{align}
Therefore $q_1$ does not appear in the r.h.s.\ of Eq.~\ref{eq:J10} and can be
omitted. Further, the last term in $J^{1,0}_{\text{fvNCE}}$ is the (Neyman)
$\chi^2$-divergence between $p_\theta(X,Z)$ and $p_n(X)q_0(Z|X)$. After some
algebraic manipulations it can be shown that $J^{1,0}_{\text{fvNCE}}$ is (up
to constants) a weighted squared distance,
\begin{align}
  J^{1,0}_{\text{fvNCE}}(\theta,q_0) \doteq \frac{1}{2} \sum_{x,z} \frac{\Paren{ p_\theta(x,z) - p_{d,0}(x,z) }^2}{p_{n,0}(x,z)}.
\end{align}
Overall the aim is to minimize the weighted squared distance between the
generative joint model $p_\theta(X,Z)$ and the data-encoder induced one
$p_d(X)q(Z|X)$. In contrary to the setting where $\alpha=0$ (or
$\alpha$ is at least small) and therefore it is natural to model
$\log p_\theta$, it seems more natural to model $p_\theta$ directly (instead
of the log-likelihood) in Eq.~\ref{eq:J10}. Hence, the choice
$\alpha=1$ is connected to density ratio
estimation~\cite{sugiyama2012density,sugiyama2012density_book}, that typically
uses shallow mixture models to represent the density ratio
$p_\theta/p_n$. In fact, $J^{1,0}_{\text{fvNCE}}$ in Eq.~\ref{eq:J10} is
closely related to least-squares importance fitting~\cite{kanamori2009least}
when $q_0=q_1$.

\section{Numerical Experiments}

In this section we illustrate the difference in the behavior of several
instances of $J_{\text{fv-NCE}}^{\alpha,\beta}$---in particular in comparison
with classical autoencoders and VAEs---on toy examples.

\begin{figure}[t]
  \centering
  \subfigure[Test inputs]{\includegraphics[width=0.245\textwidth]{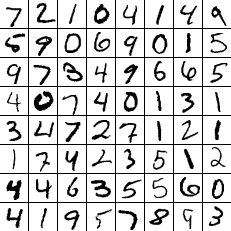}}
  \subfigure[Samples from $p_n$]{\includegraphics[width=0.245\textwidth]{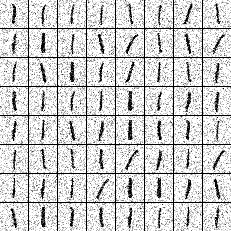}}
  \subfigure[VAE]{\includegraphics[width=0.245\textwidth]{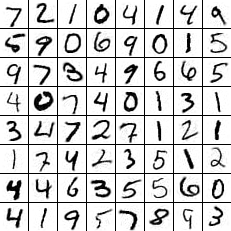}}
  \subfigure[$J_{\text{fv-NCE}}^{0,0}$]{\includegraphics[width=0.245\textwidth]{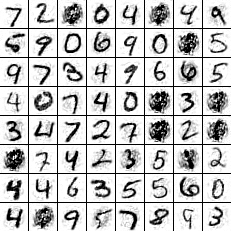}}
  \caption{
    The impact of the 2nd term in Eq.~\ref{eq:rVAE} on the reconstruction of test inputs (a).
    $p_n$ is chosen as a kernel-density estimator of several digits in a validation set showing ``1'', with samples shown in (b).
    Reconstructions of the inputs using a VAE-trained encoder-decoder are given in (c), and (d) shows 
    the corresponding reconstruction for a encoder-decoder trained using $J_{\text{fv-NCE}}^{0,0}$ (Eq.~\ref{eq:rVAE}).
    Input patches showing a ``1'' are poorly reconstructed (as intended).
  }
  \label{fig:VAE}
\end{figure}

\subsection{Noise-penalized variational autoencoders}

First, we demonstrate the capability to steer the behavior of an 784-256-784
autoencoder (with deterministic encoder) by using
$J_{\text{fv-NCE}}^{0,0}$ (Eq.~\ref{eq:rVAE}). The noise distribution
$p_n$ is a kernel density estimate of inputs depicting the digit ``1'' from a
validation set. Since the cost for false positives induced by
$-f_0^{0,0}=r$ is higher than the cost for false negatives
($-f_1^{0,0}(r)=-\log r$), anything resembling a digit ``1'' is expected to be
poorly reconstructed---even when those digits appear frequently in the
training data. Fig.~\ref{fig:VAE} visually verifies this on test inputs. This
feature of Eq.~\ref{eq:rVAE} is useful when training data for OOD detection is
contaminated by outliers, but a collection of outliers is available; or when
an autoencoder-based OOD detector is required to identify certain patterns as
OOD.

\subsection{Stronger noise penalization using $J_{\text{fv-NCE}}^{\alpha,0}$}
\label{sec:eval_alpha}

Since $f_0^{\alpha,0}$ penalizes false positives stronger than
$f_1^{\alpha,0}$ does for false negatives, we expect different solutions for
different choices of $\alpha$. With infinite data and correctly specified
models $\log p_\theta$, all PSRs will return the same solution (up to the
issue of local maxima), but we only have finite training data and clearly
underspecified models.

We fix the decoder variance to $\sigma^2_{\text{dec}}=1/8^2$ and use a kernel
density estimate with bandwidth
$\sigma_{\text{kde}}=2\sigma_{\text{dec}}$ as noise distribution
$p_n$. By setting $\alpha>0$, noise samples (which are near the training data
in this setting) force the model $p_\theta$ to explicitly concentrate on the
training data. Samples $x\sim p_n$ have a larger reconstruction error as
compared to the VAE setting ($\alpha=0$). Table~\ref{tab:alphas} lists average
decoding log-likelihoods for several values of $\alpha$. VAEs reconstruct
noise samples worse than standard autoencoders (AEs) due to their latent code
regularization. This behavior is generally amplified for increasing
$\alpha$, as the difference between the average reconstruction error grows
with $\alpha$. We also include $J^{0,1}_{\text{fvNCE}}$ (Sec.~\ref{sec:J01})
for reference, which behaves in practice similar to
VAEs. Fig.~\ref{fig:noise_reconstructions} visualizes the decreasing
reconstruction quality of samples drawn from $p_n$.

In order to avoid vanishing gradients when $\alpha>0$ in the initial training
phase, in view of Cor.~\ref{cor:closedness} we use actually a linear
combination of $J^{\alpha,0}_{\text{fvNCE}}$ (with weight $0.9$) and
$J^{0,0}_{\text{fvNCE}}$ (with weight $0.1$) as training loss.
Table~\ref{tab:alphas} lists the values for two ReLU-based MLP networks
(trained from the same random initial weights) obtained after 100
epochs. Since the log-ratios such as
$\log r = \log p_\theta(x,z) - \log p_{n,1}(x,z)$ can attain large magnitudes,
expressions such as $r^\alpha$ and $r^{\alpha+1}$ are evaluated using a
``clipped'' exponential function: we use the first-order approximation
$e^T(u-T+1)$ when $u>T$ for a threshold value $T$, which is chosen as
$T=10$ in our implementation.

\begin{table}[tb]
  \centering
  \small
  \subtable[784-128-784]{
    \begin{tabular}{cccc}
      Method/$(\alpha,\beta)$ & $x\sim p_d$ & $x\sim p_n$ & Difference \\ \hline
      AE & 766 & -1138 & 1904 \\
      VAE & 765 & -1609 & 2374 \\
      $(1/256,0)$ & 749 & -1665 & 2414 \\
      $(1/64,0)$ & 698 & -1863 & 2561 \\
      $(1/16,0)$ & 753 & -1818 & 2571 \\
      $(0,1)$ & 736 & -1656 & 2392
    \end{tabular}
  }
  \subtable[784-256-128-256-784]{
    \begin{tabular}{cccc}
      Method/$(\alpha,\beta)$ & $x\sim p_d$ & $x\sim p_n$ & Difference \\ \hline
      AE & 756 & -919 & 1675 \\
      VAE & 769 & -1373 & 2142 \\
      $(1/256,0)$ & 774 & -1402 & 2176 \\
      $(1/64,0)$ & 748 & -1520 & 2268 \\
      $(1/16,0)$ & 777 & -1463 & 2240 \\
      $(0,1)$ & 775 & -1372 & 2147
    \end{tabular}
  }
  \caption{Average log-likelihood $\log p_\theta(x|g(x))$ in nats. Higher values indicate lower reconstruction error.
  }
  \label{tab:alphas}
\end{table}

\begin{figure}[t]
  \centering
  \subfigure[Noise samples]{\includegraphics[width=0.245\textwidth]{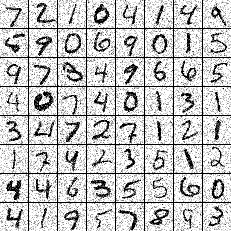}}
  \subfigure[AE: -1138 nats]{\includegraphics[width=0.245\textwidth]{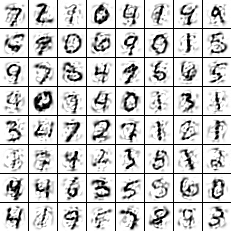}}
  \subfigure[VAE: -1609 nats]{\includegraphics[width=0.245\textwidth]{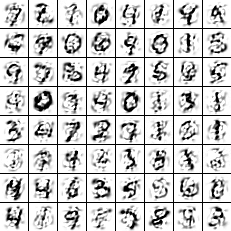}}
  \subfigure[$\alpha\!\!=\!\!\tfrac{1}{16}$: -1818 nats]{\includegraphics[width=0.245\textwidth]{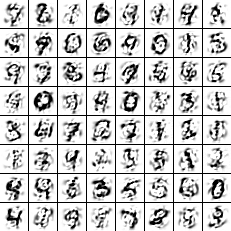}}
  \caption{Reconstruction of samples $x\sim p_n$ (a). VAEs (c) and
    $J^{\alpha,0}_{\text{fvNCE}}$ (d) increasingly force such samples to be
    poorly reconstructed compared to AEs (b), while maintaining a similar
    reconstruction error for training data $x\sim p_d$ (see Table~\ref{tab:alphas}). }
  \label{fig:noise_reconstructions}
\end{figure}

\section{Conclusion}

In this work we propose fully variational noise-contrastive estimation as a
tractable method to apply noise-contrastive estimation on latent variable
models. As with most variational inference methods, the resulting empirical
loss only needs samples from the data, noise and encoder distributions. We are
largely interested in the existence and basic properties of such framework and
unravel a connection with variational autoencoders. In light of this
connection, VAEs are now justified to be steered explicitly towards poorly
reconstructing samples from a user-specified noise distribution.

The utility of our framework for improved OOD detection and enabling general
energy-based decoder models is left as future work. Further, the highly
asymmetric nature of the classification loss suggests a potential but
yet-to-explore connection with one-class SVMs~\cite{scholkopf2001estimating}
and support vector data description~\cite{tax2004support}.

\paragraph{Acknowledgement}
This work was partially supported by the Wallenberg AI, Autonomous Systems and
Software Program (WASP) funded by the Knut and Alice Wallenberg Foundation.

\bibliographystyle{plain}
\bibliography{literature}

\end{document}